\PassOptionsToPackage{table}{xcolor}
\documentclass[11pt]{article}

\usepackage[preprint]{acl}

\usepackage{times}
\usepackage{latexsym}
\usepackage[T1]{fontenc}
\usepackage[utf8]{inputenc}
\usepackage{microtype}
\usepackage{inconsolata}
\usepackage{graphicx}

\usepackage{amsmath,amsfonts,bm}

\def\eqref#1{equation~\ref{#1}}

\def\1{\bm{1}}

\DeclareMathAlphabet{\mathsfit}{\encodingdefault}{\sfdefault}{m}{sl}
\SetMathAlphabet{\mathsfit}{bold}{\encodingdefault}{\sfdefault}{bx}{n}

\newcommand{\KL}{D_{\mathrm{KL}}}

\usepackage{url}
\usepackage{amsmath,amssymb,amsfonts,amsthm}
\usepackage{booktabs}
\usepackage{multirow}
\usepackage{array}
\usepackage{xcolor}
\usepackage{colortbl}
\usepackage{enumitem}
\usepackage{caption}
\usepackage{wrapfig}
\usepackage{tikz}
\usepackage{float}
\usepackage{algpseudocode}
\floatstyle{ruled}
\newfloat{algorithm}{t}{loa}
\floatname{algorithm}{Algorithm}

\usetikzlibrary{arrows.meta,positioning,fit,backgrounds,calc,decorations.pathreplacing}

\newcommand{\method}{\textsc{CoFiRe}}
\newcommand{\methodlong}{Confident First, Reflective Second}
\newcommand{\ctr}{\operatorname{ctr}}
\renewcommand{\KL}{D_{\mathrm{KL}}}
\providecommand{\CE}{\mathrm{CE}}
\providecommand{\Ind}{\mathbb{I}}
\providecommand{\Exp}{\mathbb{E}}
\newcommand{\trust}{\textsc{Trust}}
\newcommand{\review}{\textsc{Review}}
\newcommand{\seek}{\textsc{Seek}}
\newcommand{\multihop}{\textsc{Multi-Hop}}
\newcommand{\fix}{\textsc{fix}}
\newcommand{\brk}{\textsc{break}}
\newcommand{\cw}{\textsc{cw}}
\newcommand{\Spre}{S_{\mathrm{pre}}}
\newcommand{\Sgate}{S_{\mathrm{gate}}}
\newcommand{\Srev}{S_{\mathrm{rev}}}
\newcommand{\Acc}{\mathrm{Acc}}

\newsavebox{\resulttablebox}
\newcommand{\fitresulttable}[1]{%
  \sbox{\resulttablebox}{#1}%
  \ifdim\wd\resulttablebox>\linewidth
    \resizebox{\linewidth}{!}{\usebox{\resulttablebox}}%
  \else
    \usebox{\resulttablebox}%
  \fi
}

\newtheorem{proposition}{Proposition}
\newtheorem{observation}{Observation}

\title{Confident First, Reflective Second: Reinforcement Learning of Confidence-Preserving Reflection for Video Understanding}

\author{%
  Zihao Lin$^{1}$ \quad
  Jingyuan Liu$^{2}$ \quad
  Yucheng Zhang$^{2}$ \quad
  Yuguang Yao$^{3}$ \quad
  Lifu Huang$^{1}$ \\[4pt]
  $^{1}$University of California, Davis \quad
  $^{2}$Independent Researcher \quad
  $^{3}$Intuit \\
}

\begin{document}
\maketitle

\begin{abstract}
Video models answer from sampled frames and may miss the decisive moment.
A common remedy is to commit to an answer, re-examine the video, and confirm or revise it, a loop now mainly learned with reinforcement learning (RL).
We find that RL erodes the initial answer, a phenomenon we term \emph{confidence erosion}, where confidence means the label-verified reliability of the initial answer rather than self-reported certainty: because errors can be repaired later, a correct initial answer earns less reward, and the initial answer stalls even as the final answer's gain over it rises.
We propose the principle \emph{confident first, reflective second}: the initial answer must stay reliable on its own, above a floor set by a reference policy, and reflection may only add to it.
Our method, \method{}, implements this principle with one mechanism for each erosion channel: reward incentives, parameter drift, and gradient interference.
\method{} reaches 67.0\% mean accuracy on Video-MME, MLVU, and LongVideoBench, 1.7 points above the strongest compared 7B method, and improves initial-answer and final accuracy over outcome-reward RL by 4.1 and 3.7 points on average: a stronger initial answer carries through to a stronger final one.
\end{abstract}

\section{Introduction}

Understanding long-form videos requires reasoning over visual evidence that is often incomplete, sparse, or difficult to localize.
Under compute and context limits, multimodal large language models process only a subset of frames and may overlook decisive moments, miss fine-grained details, or misinterpret temporal relationships \citep{fu2025videomme,zhou2025mlvu,wu2024longvideobench}.
Recent video agents therefore give an initial answer, retrieve or re-examine relevant segments, and verify or revise the answer based on new evidence \citep{wang2024videoagent,pan2025timesearchr,madaan2023selfrefine,reflectr1}.
Reinforcement learning (RL) has emerged as a promising way to learn when to seek more evidence and how to use it \citep{kumar2025score,wang2025vlrethinker,reflectr1}.
Unlike reflection within a single reasoning trace \citep{guo2025deepseekr1,feng2025videor1}, an explicit answer-and-revision process exposes both the initial and final answers, so their correctness and the outcome of each revision can be evaluated separately.

\begin{figure}[t]
\centering
\begin{tikzpicture}[x=0.74cm,y=1.5cm,font=\scriptsize]
\draw[gray!45,dashed] (0,1.8)--(1.8,0);
\draw[gray!45,dashed] (1.8,2.0)--(3.8,0) node[pos=0,above,text=gray]{Final 60};
\draw[gray!45,dashed] (3.8,2.0)--(5.8,0) node[pos=0,above,text=gray]{Final 62};
\draw[gray!45,dashed] (5.8,2.0)--(7.8,0) node[pos=0,above,text=gray]{Final 64};
\draw[gray!70,densely dotted] (0,0.8)--(8,0.8);
\draw[->] (0,0)--(8.2,0);
\draw[->] (0,0)--(0,2.15);
\foreach \x/\l in {1/58,3/60,5/62,7/64}{\draw (\x,0)--(\x,-0.04) node[below]{\l};}
\foreach \y/\l in {0.3/{$-0.5$},0.8/{$0.0$},1.3/{$0.5$},1.8/{$1.0$}}{\draw (0,\y)--(-0.08,\y) node[left]{\l};}
\node at (4.1,-0.36) {Initial-answer accuracy $\Acc_{\mathrm{pre}}$ (\%)};
\node[rotate=90] at (-1.2,1.0) {Reflection gain $\Delta$ (points)};
\foreach \px/\py in {4.0/1.4,5.3/0.9,5.4/0.6,5.3/0.8,5.0/1.1,6.2/1.0,6.5/1.3}{\fill[gray!55] (\px,\py) circle (1.1pt);}
\fill[black] (0.7,0.3) circle (1.7pt) node[right]{Zero-shot};
\fill[black] (3.4,0.5) circle (1.7pt) node[below]{SFT};
\fill[orange!85!black] (3.3,1.3) circle (1.9pt) node[left]{Outcome RL};
\fill[red!75!black] (7.2,1.1) circle (2.3pt) node[above,xshift=-20pt]{\textbf{\method{} (ours)}};
\draw[-{Stealth[length=4pt]},orange!85!black,thick] (3.4,0.56)--(3.32,1.22);
\draw[-{Stealth[length=4pt]},red!75!black,thick] (3.5,0.51)--(7.08,1.08);
\end{tikzpicture}
\vspace{-2mm}
\caption{\textbf{Reflection gain is not reflective progress.} Video-MME policies by initial-answer accuracy and reflection gain $\Delta$; dashed lines are iso-lines of final accuracy. From the same SFT policy, Outcome RL moves \emph{up} (more gain, same initial answer), while \method{} moves \emph{right} to a higher final accuracy. Gray dots: configurations from Table~\ref{tab:ablation_full}.}
\label{fig:plane}
\end{figure}
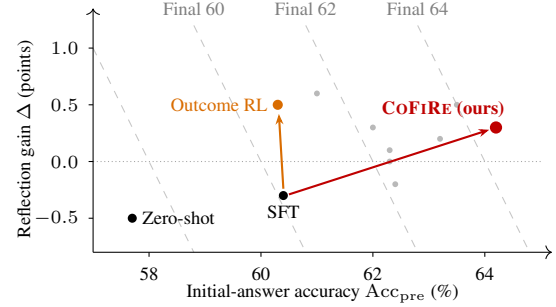

Existing reflective objectives mainly reward final-answer correctness or improvement through revision, leaving open how optimizing reflection affects initial-answer accuracy.
Answering and evidence-conditioned revision are distinct capabilities.
The former reasons over available observations, whereas the latter must identify missing or conflicting evidence and judge whether to revise.
Reflection can repair initial mistakes but is imperfect: evidence may be unavailable, retrieval may miss relevant moments, and new observations can be misinterpreted or even overturn a correct initial answer.
A robust reflective agent should therefore learn not only to correct mistakes but also to answer accurately before correction is needed.

We identify an imbalance between initial-answer learning and correction in reflective RL.
As illustrated in Figure~\ref{fig:plane}, outcome-based RL (Outcome RL) raises the reflection gain on Video-MME from $-0.3$ to $+0.5$ points while initial-answer accuracy stays nearly unchanged (60.4\% $\rightarrow$ 60.3\%), and similar patterns emerge on MLVU and LongVideoBench.
A larger reflection gain thus need not indicate a stronger system and may instead reflect greater dependence on correction.
We refer to this imbalance, particularly when accompanied by degradation of the initial answer, as \emph{confidence erosion}, where \emph{confidence} denotes the label-verified reliability of the initial answer rather than self-reported certainty.

This imbalance arises from two challenges in optimizing a shared reflective policy.
First, outcome-based rewards can obscure the contribution of the initial answer (\emph{reward incentives}), since an incorrect initial answer that is later repaired can receive the same credit as a correct one that remains unchanged.
As we show in Section~\ref{sec:erosion}, under fixed repair and break rates, successful correction reduces the marginal value of initial-answer accuracy, and correction bonuses can further favor repaired trajectories.
Stage-specific credit can separate the two \citep{kumar2025score,reflectr1}, but reward attribution alone does not resolve the problem.
Second, answering and reflecting share parameters, so updates that improve retrieval or revision may alter reliable initial-answer decisions (\emph{parameter drift}), and reflective gradients may conflict with those improving the initial answer (\emph{gradient interference}).

We therefore propose \method{} (\methodlong{}), which protects initial-answer accuracy while encouraging beneficial revision, rather than treating reflection gain as an end in itself.
We formalize this principle as maximizing initial-answer reliability and reflective utility subject to a \emph{confidence floor} relative to a frozen reference policy, and approximately optimize it with three mechanisms.
(1)~\emph{Initial-answer credit} rewards the initial answer for its own correctness, separately rewards corrections and penalizes destructive revisions, and raises its weight through an adaptive multiplier when initial-answer accuracy falls below the floor.
(2)~\emph{Verified-decision anchoring} consolidates reference decisions verified as correct through label-guided supervision and decision-local regularization, stabilizing reliable predictions while allowing incorrect ones to improve.
(3)~\emph{Asymmetric gradient protection} prioritizes the protected objective whenever the initial-answer and reflective gradients conflict.

We evaluate \method{} with Qwen2.5-VL-7B \citep{bai2025qwen25vl} on Video-MME \citep{fu2025videomme}, MLVU \citep{zhou2025mlvu}, and LongVideoBench \citep{wu2024longvideobench}.
Under identical initialization and inference, \method{} improves initial-answer and final accuracy over Outcome RL by 4.1 and 3.7 points on average, so its initial-answer gains carry through to the final answer, and reflection still helps on all three benchmarks.
It also reaches 67.0\% average accuracy, 1.7 points above the strongest compared 7B method.
Ablations show that the three mechanisms contribute complementarily, with notable gains on long videos.
Reflective learning should thus be judged not only by how well it corrects mistakes but also by whether it improves the answers that correction starts from.

Our main contributions are threefold:
\begin{itemize}[leftmargin=*,nosep]
\item \textbf{Understanding the limitations of outcome-driven reflective learning.} We show that gains from reflective RL can concentrate in revision while initial-answer accuracy stays largely unchanged, and analyze how reward attribution and shared parameters cause this imbalance.
\item \textbf{A unified framework for reliable reflective optimization.} We introduce \method{}, which combines three mechanisms to improve both the initial answer and its correction.
\item \textbf{Empirical validation on video understanding.} Experiments on three video benchmarks show gains in both initial-answer and final accuracy, supported by ablations and analyses.
\end{itemize}

\section{Related Work}\label{sec:related}

\paragraph{Answering and re-examining video.}
Outcome-reward RL with group-relative advantages is the standard recipe for eliciting reasoning in LLMs \citep{shao2024deepseekmath,guo2025deepseekr1} and multimodal models \citep{huang2025visionr1}.
For video, Video-R1 rewards temporal reasoning \citep{feng2025videor1} and VideoChat-R1 reinforces spatio-temporal perception \citep{li2025videochatr1}, but both answer in a single pass.
A second line lets the model look again: training-free agents seek question-relevant evidence \citep{wang2024videoagent,wang2025videotree,ye2025tstar}, TimeSearch-R interleaves temporal search with reasoning under a self-verification reward \citep{pan2025timesearchr}, and Reflect-R1 arbitrates an intuitive answer against retrieved evidence \citep{reflectr1}.
Most trained variants reward only the \emph{final} answer, which cannot distinguish an initially correct answer from a repaired one and reduces the initial answer to a draft. Stage-decoupled credit (Reflect-R1) removes this discount but still trains the initial answer only as an input to revision. We instead make initial-answer accuracy an explicit, protected training target.

\paragraph{Learning to revise.}
Self-Refine and Reflexion improve outputs with generated feedback \citep{madaan2023selfrefine,shinn2023reflexion}, yet intrinsic self-correction fails without reliable feedback \citep{huang2024cannot}.
SCoRe and S$^2$R train self-correction with RL \citep{kumar2025score,ma2025s2r}, and SCoRe adds a first-attempt KL penalty and a progress reward. R$^3$V, VL-Rethinker, SRPO, and ReCoT extend reflection to vision-language models \citep{cheng2025r3v,wang2025vlrethinker,wan2025srpo,qu2025recot}, and Reflect-R1 additionally penalizes overturning correct answers \citep{reflectr1}.
These methods optimize whether the \emph{revision} improves or preserves an answer, whereas we show that this objective erodes the \emph{initial answer}, an effect that neither per-rollout revision penalties nor fixed first-attempt regularizers correct.
Appendix~\ref{app:related} relates \method{} to constrained RL, continual learning, and gradient surgery.

\section{Confidence Erosion in Learning to Re-examine}\label{sec:erosion}

\subsection{Answering and Re-examining a Video}\label{sec:setup}

\begin{figure*}[t]
\centering
\includegraphics[width=\textwidth]{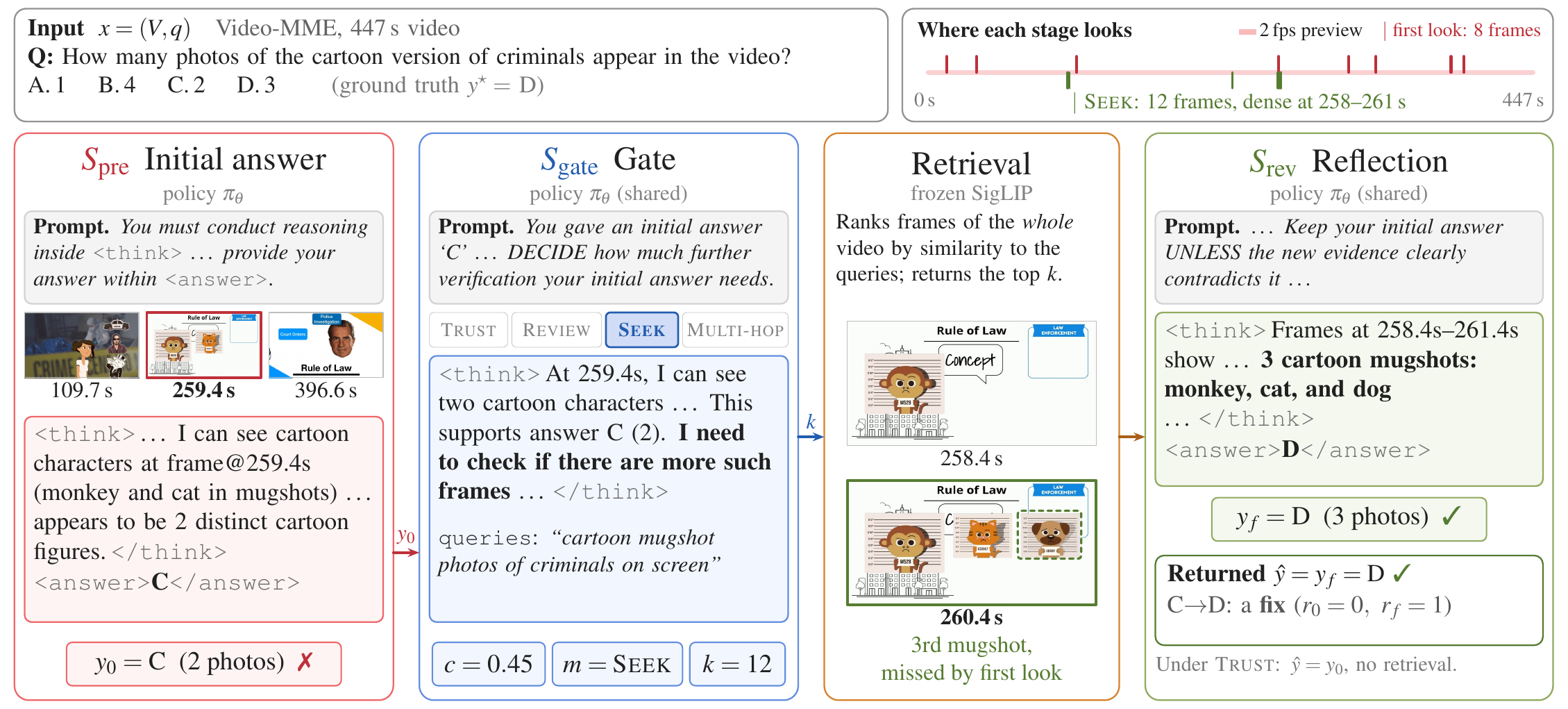}
\caption{\textbf{Answering and re-examining a Video-MME example.} From a 2\,fps preview and an eight-frame first look showing only two mugshots, $\Spre$ answers $y_0=$ C. $\Sgate$ reports low confidence and chooses \seek{}; frames at 258--261\,s reveal a third mugshot, and $\Srev$ revises to $y_f=$ D, a \fix{}. More examples are in Appendix~\ref{app:cases}.}
\label{fig:example}
\end{figure*}

Let $x=(V,q)$ be a video-question input with answer options and correct option $y^\star$. As illustrated in Figure~\ref{fig:example}, a policy $\pi_\theta$ with shared parameters $\theta$ runs three stages. The initial-answer stage $\Spre$ observes a uniform 2\,fps preview of the whole video, whose fixed visual budget renders long videos at low resolution. It therefore requests a \emph{first look} of up to eight higher-resolution frames, selected for query relevance and temporal diversity, and commits to an initial answer $y_0$.

The gate $\Sgate$ inspects the first-look frames, reports a confidence score, and selects a mode $m$: \trust{} (return $y_0$), \review{} (re-examine the first look), \seek{} (retrieve once), or \multihop{} (up to three retrieval rounds), together with retrieval queries and a frame count when needed. A frozen SigLIP encoder \citep{zhai2023siglip} ranks frames from a 2\,fps cache of the whole video by query similarity and option-discriminating scores, never returning a frame twice. The post-answer stage $\Srev$ then takes the initial-answer reasoning, the first-look and retrieved frames, and a frame-derived option belief (a distribution over options from SigLIP frame--option similarity), and produces a candidate answer $y_f$, which may confirm $y_0$. The preview is dropped after $\Spre$ to bound the context.

At inference, the returned answer is $\hat y=y_0$ under \trust{} and $\hat y=y_f$ otherwise; during training, $y_f$ is generated in every mode. We refer to $\Sgate$ and $\Srev$ together as \emph{reflection}. Because $y_0$ is explicitly committed, the task label alone determines whether the initial answer is correct and whether reflection repairs or preserves it, without step-level annotation.

\subsection{Why Reflective Training Erodes Confidence}\label{sec:metrics}\label{sec:incentive}\label{sec:why}

Let $r_0$ and $r_f$ indicate whether $y_0$ and $y_f$ are correct. A \fix{} repairs an error ($r_0=0$, $r_f=1$), a \brk{} overturns a correct answer ($r_0=1$, $r_f=0$), and a \cw{} switches between incorrect options. We report initial-answer accuracy $\Acc_{\mathrm{pre}}=\Exp[r_0]$, final accuracy $\Acc(\hat y)$, and reflection gain $\Delta=\Acc(\hat y)-\Acc_{\mathrm{pre}}$. Let $a=\Exp[r_0]$, and let $\rho$ and $\delta$ be the rates at which reflection repairs incorrect and breaks correct initial answers.
\begin{observation}[Reflection discounts the initial answer]\label{prop:hazard}
Since $\Pr[\fix]=\rho(1-a)$ and $\Pr[\brk]=\delta a$,
\begin{equation}
\begin{aligned}
\Exp[r_f]&=\rho+(1-\rho-\delta)\,a,\\
\Exp[r_f-r_0]&=\rho-(\rho+\delta)\,a.
\end{aligned}
\label{eq:hazard}
\end{equation}
\end{observation}
For fixed reflective behavior, $\partial\Exp[r_f]/\partial a=1-\rho-\delta$. An outcome reward thus values a correct initial answer at only $1-\rho-\delta$, so the better the reflection, the weaker the pressure to be right. An improvement reward even values it \emph{negatively}, paying more for an error that is later fixed than for an answer that is correct throughout. Since reflection gain decreases in $a$, selecting policies by reflection gain also selects for erosion.
Figure~\ref{fig:plane} confirms this prediction: Outcome RL turns a negative reflection gain positive while initial-answer accuracy stays flat, and Table~\ref{tab:main} shows the same pattern on all three benchmarks. The policy learns a better fallback, not a better answer.
\paragraph{Drift and interference.}
Routing credit by stage \citep{kumar2025score,reflectr1} removes this discount but not erosion, because answering and reflecting share parameters. Reflection updates can oppose the initial-answer gradient (\emph{interference}) and shift answers on unsampled questions (\emph{drift}), and neither a fixed reward weight nor the model's self-reported confidence (Section~\ref{sec:analysis}) can detect these effects.

\section{\methodlong{}}\label{sec:method}
Figure~\ref{fig:rwr_framework} gives an overview of \method{}. Section~\ref{sec:objective} formalizes the \emph{confident first, reflective second} principle as a constrained objective, and Sections~\ref{sec:credit}--\ref{sec:protect} close the incentive, drift, and interference channels in turn.

\begin{figure*}[t]
\centering
\includegraphics[width=\textwidth]{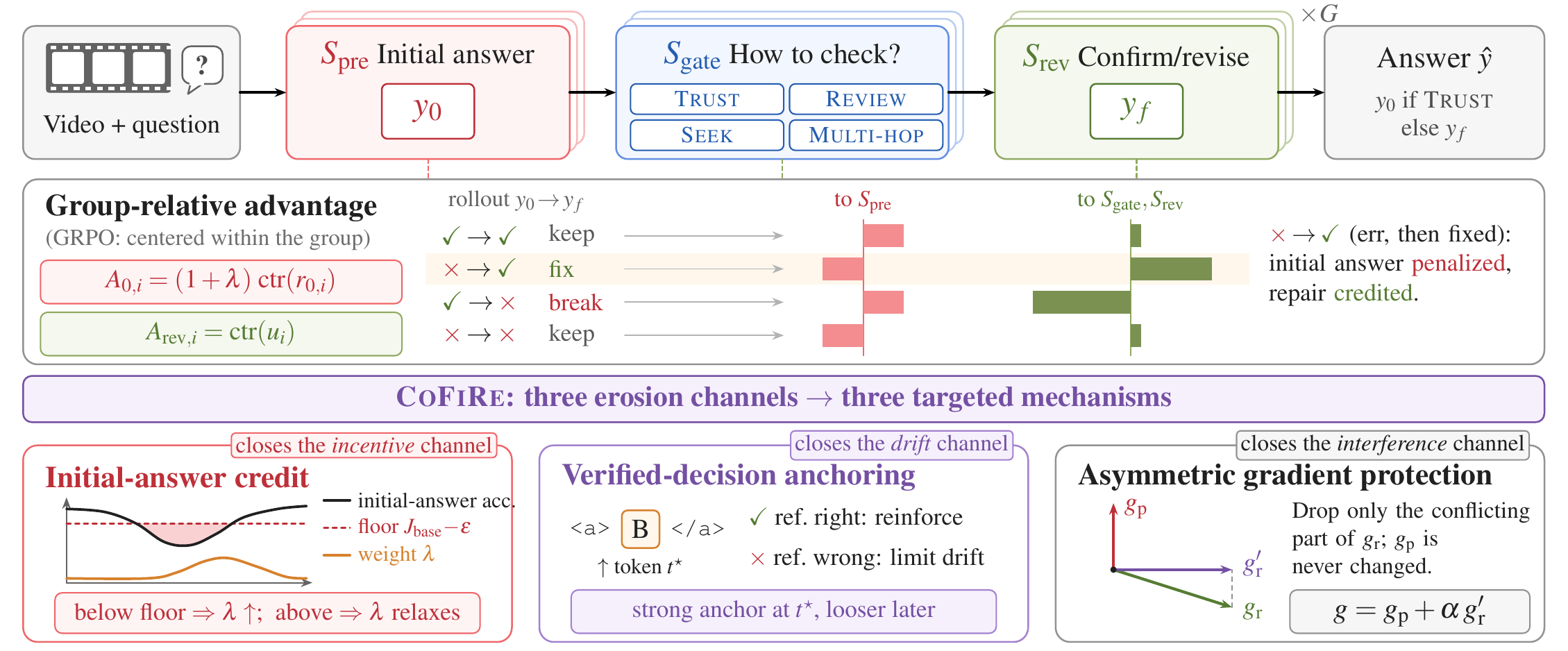}
\caption{\textbf{Overview of \method{}.} \emph{Top:} $G$ rollouts of one policy answer ($\Spre$), choose how to check ($\Sgate$), and confirm or revise ($\Srev$). \emph{Middle:} stage-routed advantages penalize the initial answer of an err-then-fix rollout but credit the repair. \emph{Bottom:} one mechanism per erosion channel: $\lambda$ rises below the initial-answer accuracy floor (incentive), verified reference answers anchor the decision token (drift), and only the conflicting part of the reflection gradient is removed (interference).}
\label{fig:rwr_framework}
\end{figure*}

\subsection{Confident First, Reflective Second}\label{sec:objective}

A frozen pretrained model $\pi_{\mathrm{base}}$ provides reference answers, and a frozen copy of the supervised initialization, $\pi_{\mathrm{ref}}$, provides KL targets. For each training input $x$, we cache the direct reference answer $y_b$ and its correctness $r_b=\Ind[y_b=y^\star]$, which separate verified reference decisions from reference errors.

Let $\mathcal D_w$ denote the training distribution over labeled inputs, including reference-error reweighting when enabled. Define initial-answer reliability $J_0(\theta)=\Exp[r_0]$, reference reliability $J_{\mathrm{base}}=\Exp[r_b]$, and reflective utility $J_{\mathrm{rev}}(\theta)=\Exp[u]$, where $u$ is the transition utility in Equation~\ref{eq:utility}. Expectations are over inputs from $\mathcal D_w$ and, for $J_0$ and $J_{\mathrm{rev}}$, over trajectories from $\pi_\theta$.

\method{} is built on the constrained objective
\begin{equation}
\begin{aligned}
\max_\theta\quad &J_0(\theta)+\alpha J_{\mathrm{rev}}(\theta)\\
\text{subject to}\quad &J_0(\theta)\ge J_{\mathrm{base}}-\epsilon,
\end{aligned}
\label{eq:constrained}
\end{equation}
where $\alpha\ge0$ weights reflection and $\epsilon\ge0$ is the permitted shortfall relative to the reference ($\epsilon=0.10$). We call the constraint the \emph{confidence floor}.

Equation~\ref{eq:constrained} differs from outcome-driven reflective RL in two ways. First, initial-answer reliability enters as a separate term, so reflection is never credited for the value of the answer. Second, it enters as a floor, so the policy cannot trade it for reflective utility, however profitable reflection becomes. Reflection gain $\Delta$ is deliberately absent, since by Observation~\ref{prop:hazard} optimizing it would reward erosion. We optimize Equation~\ref{eq:constrained} approximately with stage-routed policy surrogates (Appendix~\ref{app:impl}).

\subsection{Initial-Answer Credit under a Confidence Floor}\label{sec:credit}

\paragraph{Credit the initial answer only for itself.}
For each question, sample $G$ trajectories indexed by $i=1,\ldots,G$. For a scalar trajectory signal $z_i$, let the within-group centering be $\ctr(z_i)=z_i-\frac{1}{G}\sum_{j=1}^{G}z_j$. Let $\lambda\ge0$ be an adaptive floor multiplier. We define
\begin{align}
A_{0,i}&=(1+\lambda)\ctr(r_{0,i}),\label{eq:a1}\\
u_i&=\Ind[\fix_i]-\beta\Ind[\brk_i]-\gamma\Ind[\cw_i],\label{eq:utility}\\
A_{\mathrm{rev},i}&=\ctr(u_i).\nonumber
\end{align}
We assign $A_{0,i}$ to initial-answer tokens and $A_{\mathrm{rev},i}$ to gate and post-answer tokens, and apply $\alpha$ when combining the stage gradients (Section~\ref{sec:protect}).
Under this routing, reflection no longer pays for initial-answer errors. A trajectory that errs and is later fixed gives its initial answer a below-average advantage, while its reflective stages are credited for the repair. In the notation of Observation~\ref{prop:hazard}, the initial answer is valued at $1+\lambda$ per unit of accuracy, independent of $\rho$ and $\delta$.
The transition utility is asymmetric. With $\beta=1.5$ and $\gamma=0.2$, a fix earns $+1$, a break costs $-1.5$, and a change between incorrect options costs $-0.2$, so reflection learns to respect a reliable initial answer rather than compete with it. Format rewards and a mode-dependent effort cost further encourage valid and economical verification (Appendix~\ref{app:impl}).

\paragraph{Raise the weight when reliability slips.}
Routed credit removes the explicit discount, but drift and interference can still erode the initial answer, and a fixed credit weight cannot react. We therefore adapt the weight through the Lagrangian of Equation~\ref{eq:constrained},
\begin{equation}
\begin{aligned}
\mathcal J_\lambda(\theta)
={}&(1+\lambda)J_0(\theta)+\alpha J_{\mathrm{rev}}(\theta)\\
&-\lambda(J_{\mathrm{base}}-\epsilon),
\end{aligned}
\label{eq:lagrangian}
\end{equation}
which yields the weight $1+\lambda$ in Equation~\ref{eq:a1}. With $\widehat J_0$ and $\widehat J_{\mathrm{base}}$ denoting moving-average estimates of initial-answer and cached reference correctness on training batches (Appendix~\ref{app:accuracy_estimation}), we update the multiplier by projected dual ascent,
\begin{equation}
\lambda\leftarrow\mathrm{clip}\!\left(\lambda+\eta_\lambda(\widehat J_{\mathrm{base}}-\epsilon-\widehat J_0),0,\lambda_{\max}\right),
\label{eq:dual}
\end{equation}
with step size $\eta_\lambda>0$ and bound $\lambda_{\max}$, holding $\lambda$ fixed within each policy update.

Below the floor, the weight on being right the first time grows until the violation is repaired; above it, the extra weight relaxes and reflection receives a larger share of the update. Unlike a fixed first-attempt regularizer, the protection strength thus adapts to the policy's own behavior.

\begin{table*}[t]
\centering
\caption{\textbf{Overall performance on three video benchmarks.} Video-MME is evaluated without subtitles; MLVU reports m-avg and LongVideoBench validation accuracy. $^\ddagger$ marks source subtotal discrepancies (Appendix~\ref{app:reporting}).}
\label{tab:sota}
\small
\setlength{\tabcolsep}{4pt}
\fitresulttable{%
\begin{tabular}{l|cccc|cc|c}
\toprule
\multirow{2}{*}{Model}
& \multicolumn{4}{c|}{Video-MME}
& \multirow{2}{*}{MLVU}
& \multirow{2}{*}{LongVideoBench}
& \multirow{2}{*}{Avg}\\
& Short & Medium & Long & Overall & & & \\
\midrule
\multicolumn{8}{l}{\textit{Without tools}}\\
Qwen2.5-VL-7B \citep{bai2025qwen25vl} & 71.4 & 60.1 & 52.3 & 61.3 & 57.9 & 53.9 & 57.7\\
VL-Rethinker \citep{wang2025vlrethinker} & 70.1 & 60.8 & 53.0 & 61.3 & 63.5 & 55.6 & 60.1\\
Vision-R1-7B \citep{huang2025visionr1} & 71.2 & 60.2 & 52.8 & 61.4 & 68.2 & 52.8 & 60.8\\
Video-R1-7B$^\ddagger$ \citep{feng2025videor1} & 72.2 & 58.1 & 52.3 & 60.8 & 68.5 & 60.1 & 63.1\\
\midrule
\multicolumn{8}{l}{\textit{Tool-augmented reasoning}}\\
TimeSearch-R-7B \citep{pan2025timesearchr} & 73.5 & \textbf{61.2} & 53.4 & 62.7 & 69.3 & 61.2 & 64.4\\
Reflect-R1-7B \citep{reflectr1} & 73.9 & 61.0 & 55.6 & 63.5 & 69.8 & 62.5 & 65.3\\
\rowcolor{gray!12}
\method{}-7B (ours) & \textbf{75.0} & \textbf{61.2} & \textbf{57.3} & \textbf{64.5} & \textbf{71.2} & \textbf{65.2} & \textbf{67.0}\\
\bottomrule
\end{tabular}
}
\end{table*}

\subsection{Verified-Decision Anchoring}\label{sec:anchor}

\paragraph{Consolidate verified decisions where they are made.}
The multiplier reacts to aggregate reliability, whereas drift acts on individual decisions. We therefore anchor the decision itself, and only where its correctness is verified. Let $h_t$ denote the multimodal context preceding generated position $t$. At the initial-answer option-label position $t^\star$, define $p_\theta^\star=\pi_\theta(\cdot\mid h_{t^\star})$ and $p_{\mathrm{ref}}^\star=\pi_{\mathrm{ref}}(\cdot\mid h_{t^\star})$ on the same sampled context. With $y_b$ as the target label and $\pi_{\mathrm{ref}}$ as the KL anchor,
\vspace{-2mm}
\begin{equation}
\mathcal L_{\mathrm{dec}}=\Exp\!\left[\mu r_b\,\CE(p_\theta^\star,y_b)+\kappa_1\KL(p_\theta^\star\|p_{\mathrm{ref}}^\star)\right],
\label{eq:dec}
\end{equation}
where $\mu,\kappa_1\ge0$ weight the two terms.
The correctness gate $r_b$ makes consolidation selective. The cross-entropy term consolidates the reference decision only where it is correct, so the policy is never anchored to a wrong answer and the correctness reward remains free to find a better one, while the KL term bounds drift. The preserved confidence is thus earned from labels rather than read from self-reported certainty, and no confidence threshold is needed. Because both terms act only at $t^\star$, they protect the decision without constraining the reasoning that leads to it.

\paragraph{Give reflection a looser leash.}
For stage $s\in\{\mathrm{gate},\mathrm{rev}\}$ with generated positions $\mathcal T_s$, we regularize later stages with the token-averaged KL $\mathcal L_{\mathrm{late}}=\kappa_2\sum_{s}\Exp\bigl[\frac{1}{|\mathcal T_s|}\sum_{t\in\mathcal T_s}\KL(\pi_\theta(\cdot\mid h_t)\|\pi_{\mathrm{ref}}(\cdot\mid h_t))\bigr]$.
We set $\kappa_2<\kappa_1$: the answer decision is held close to the reference, whereas gate and reflection tokens, which must learn new behavior, are only loosely regularized.

\begin{table*}[t]
\centering
\caption{\textbf{Initial-answer accuracy versus reflection gain under a shared inference protocol.} $\Acc_{\mathrm{pre}}$ and Final are initial-answer and returned-answer accuracy; $\Delta$ is their difference in points. Duration breakdowns appear in Table~\ref{tab:training_duration}.}
\label{tab:main}
\small
\setlength{\tabcolsep}{5pt}
\fitresulttable{%
\begin{tabular}{l|ccc|ccc|ccc}
\toprule
\multirow{2}{*}{Training configuration}
& \multicolumn{3}{c|}{Video-MME}
& \multicolumn{3}{c|}{MLVU}
& \multicolumn{3}{c}{LongVideoBench}\\
& $\Acc_{\mathrm{pre}}$ & Final & $\Delta$
& $\Acc_{\mathrm{pre}}$ & Final & $\Delta$
& $\Acc_{\mathrm{pre}}$ & Final & $\Delta$\\
\midrule
Zero-shot, our protocol
& 57.7 & 57.2 & $-0.5$
& 56.0 & 55.2 & $-0.8$
& 51.1 & 49.9 & $-1.2$\\

Supervised initialization
& 60.4 & 60.1 & $-0.3$
& 65.9 & 65.5 & $-0.4$
& 58.3 & 57.6 & $-0.7$\\

\quad + Outcome RL
& 60.3 & 60.8 & $+0.5$
& 65.9 & 67.8 & $+1.9$
& 57.9 & 61.2 & $+3.3$\\

\rowcolor{gray!12}
\quad + \method{}
& \textbf{64.2} & \textbf{64.5} & $+0.3$
& \textbf{69.9} & \textbf{71.2} & $+1.3$
& \textbf{62.3} & \textbf{65.2} & $+2.9$\\
\bottomrule
\end{tabular}
}
\end{table*}

\subsection{Asymmetric Gradient Protection}\label{sec:protect}

Initial-answer credit and anchoring define what must be protected, and the interference channel decides whether reflective learning respects it. We therefore make the shared update asymmetric.

Let $\mathcal L_{\mathrm{pre}}$, $\mathcal L_{\mathrm{gate}}$, and $\mathcal L_{\mathrm{rev}}$ denote the negative clipped GRPO-style policy surrogates for $\Spre$, $\Sgate$, and $\Srev$, using the routed advantages above and auxiliary shaping terms (Appendices~\ref{app:impl} and \ref{app:surrogates}). We define the protected objective $\mathcal L_{\mathrm p}=\mathcal L_{\mathrm{pre}}+\mathcal L_{\mathrm{dec}}$ and the reflective objective $\mathcal L_{\mathrm r}=\mathcal L_{\mathrm{gate}}+\mathcal L_{\mathrm{rev}}+\mathcal L_{\mathrm{late}}$, so that $\mathcal L_{\mathrm p}$ both learns correct initial answers from current rollouts and consolidates verified reference decisions.

Let $g_{\mathrm p}=\nabla_\theta\mathcal L_{\mathrm p}$ and $g_{\mathrm r}=\nabla_\theta\mathcal L_{\mathrm r}$ on the same batch and parameters. We compute
\begin{align}
\omega&=\min\!\left(0,\frac{\langle g_{\mathrm r},g_{\mathrm p}\rangle}{\|g_{\mathrm p}\|^2+\varepsilon_{\mathrm{proj}}}\right),\label{eq:projection}\\
g_{\mathrm r}'&=g_{\mathrm r}-\omega g_{\mathrm p},\qquad g=g_{\mathrm p}+\alpha g_{\mathrm r}'.\nonumber
\end{align}
Here $\omega\le0$ is the correction coefficient and $\varepsilon_{\mathrm{proj}}>0$ a numerical stabilizer. When the two gradients agree, reflection is learned unchanged. When they conflict, the opposing component of the reflection gradient is removed while $g_{\mathrm p}$ is kept intact. With exact gradients and $\varepsilon_{\mathrm{proj}}=0$, this is the Euclidean projection of $g_{\mathrm r}$ onto the half-space $\{v:\langle g_{\mathrm p},v\rangle\ge0\}$, and $\langle g_{\mathrm p},g\rangle\ge\|g_{\mathrm p}\|^2$ for $\alpha\ge0$, so a sufficiently small step decreases the protected objective at least as fast as without reflection (Proposition~\ref{prop:protect}, Appendix~\ref{app:proof}). The geometry follows A-GEM \citep{chaudhry2019agem}, but what we protect is the current initial-answer reliability of the same policy rather than a memory of past tasks.

\paragraph{Protection does not freeze wrong answers.}
The coefficient $\omega$ is computed per batch, and $r_b$ gates only the CE anchor, not the projection. Because $\ctr(r_{0,i})$ is negative for incorrect initial answers, $g_{\mathrm p}$ both raises the probability of correct initial answers and lowers that of incorrect ones. Asymmetric protection preserves both directions: it forbids reflection from undoing the push away from a wrong initial answer as well as the pull toward a right one, so the protected direction points toward better answers rather than toward the status quo.

\paragraph{Implementation as loss reweighting.}
Since computing both full-model gradients separately is expensive, we estimate $\widehat\omega$ on 16 top-layer parameter tensors, detach it, and minimize
\begin{equation}
\mathcal L_{\mathrm{update}}=(1-\alpha\widehat\omega)\mathcal L_{\mathrm p}+\alpha\mathcal L_{\mathrm r}.
\label{eq:update}
\end{equation}
With $\widehat\omega$ held constant, the gradient of Equation~\ref{eq:update} is $g_{\mathrm p}+\alpha(g_{\mathrm r}-\widehat\omega g_{\mathrm p})$, exactly the protected update of Equation~\ref{eq:projection}. The projection thus reduces to a single scalar that up-weights the initial-answer objective whenever reflection conflicts with it. The subset only affects the coefficient estimate, and the update still trains all parameters with a single backward pass (Algorithm~\ref{alg:rwr}, Appendix~\ref{app:proof}).

\begin{table*}[t]
\centering
\caption{\textbf{Contribution of each mechanism on Video-MME.} Mechanisms are added cumulatively to Outcome RL and grouped by the channel they close; the last row removes gradient protection from the full configuration. Long is final accuracy on long videos. The complete progression is in Table~\ref{tab:ablation_full}.}
\label{tab:ablation}
\small
\setlength{\tabcolsep}{5pt}
\fitresulttable{%
\begin{tabular}{llcccc}
\toprule
Channel & Configuration & $\Acc_{\mathrm{pre}}$ & Final & $\Delta$ & Long\\
\midrule
-- & Outcome RL & 60.3 & 60.8 & $+0.5$ & 51.8\\
Incentive & + routed initial-answer credit & 61.0 & 61.6 & $+0.6$ & 52.7\\
Drift & \quad + CE anchor at $t^\star$ & 62.3 & 62.4 & $+0.1$ & 54.0\\
Incentive + drift & \quad + floor multiplier, KL at $t^\star$, transition utility & 62.0 & 62.3 & $+0.3$ & 53.8\\
Interference & \quad + asymmetric gradient protection & 63.2 & 63.4 & $+0.2$ & 55.8\\
\rowcolor{gray!12}
All & \quad + later-stage KL, reference-error sampling (full \method{}) & \textbf{64.2} & \textbf{64.5} & $+0.3$ & \textbf{57.3}\\
\midrule
Interference & Full \method{} w/o gradient protection & 63.5 & 64.0 & $+0.5$ & 55.8\\
\bottomrule
\end{tabular}
}
\end{table*}

\section{Experiments}\label{sec:experiments}

We evaluate whether \method{} improves video understanding (Section~\ref{sec:results}), whether and through which mechanisms it prevents confidence erosion (Section~\ref{sec:erosion_exp}), and how reflection behaves on a confident initial answer (Section~\ref{sec:analysis}).

\subsection{Experimental Setup}\label{sec:exp_setup}
All trained variants start from the same supervised Qwen2.5-VL-7B-Instruct policy \citep{bai2025qwen25vl} and train on the 42{,}602 multiple-choice questions of the Reflect-R1 pool \citep{reflectr1} for up to 1{,}000 steps. Appendices~\ref{app:data} and \ref{app:impl} detail the data, inference protocol, and hyperparameters.
\emph{Outcome RL}, the standard reflective RL recipe analyzed in Observation~\ref{prop:hazard}, trains the same staged policy with final correctness plus a correction bonus applied to the whole trajectory, and otherwise shares the initialization, data, rollout and inference protocol, and hyperparameters of \method{} (Appendix~\ref{app:impl}).
Following Reflect-R1 \citep{reflectr1}, we evaluate on Video-MME without subtitles, MLVU m-avg, and LongVideoBench validation \citep{fu2025videomme,zhou2025mlvu,wu2024longvideobench}, reporting initial-answer accuracy $\Acc_{\mathrm{pre}}$, final accuracy $\Acc(\hat y)$, reflection gain $\Delta$, and the unweighted benchmark average Avg (Appendix~\ref{app:reporting}).

\subsection{Main Results}\label{sec:results}
As shown in Table~\ref{tab:sota}, \method{} achieves the best results among compared 7B models on every benchmark, with a 67.0\% average.
It outperforms Reflect-R1, the strongest reflective baseline, by 1.0--2.7 points, with the largest margins on long Video-MME videos and LongVideoBench, where an unreliable first glance is most costly.
These results show that protecting the initial answer pays off most in exactly the settings that motivate re-examination.

\subsection{Does \method{} Prevent Confidence Erosion?}\label{sec:erosion_exp}\label{sec:ablation}

\paragraph{\method{} improves the answer, not the fallback.}
As illustrated in Figure~\ref{fig:plane}, Outcome RL channels its gains into reflection.
In contrast, Table~\ref{tab:main} shows that \method{} raises Video-MME initial-answer accuracy by 3.9 points and final accuracy by 3.7 points with a \emph{smaller} reflection gain ($+0.3$), so ranking policies by reflection gain would select the wrong one.
Reflection still helps on every benchmark, most on LongVideoBench ($+2.9$), where a single preview is least sufficient.
This confirms that reflection gain misrepresents progress, since the improvement of \method{} comes from a stronger initial answer that carries through to the final answer.

\paragraph{Closing the incentive channel.}
Routed initial-answer credit directly tests Observation~\ref{prop:hazard}.
As shown in Table~\ref{tab:ablation}, once the initial answer is no longer paid for errors that reflection repairs, initial-answer and final accuracy rise by 0.7 and 0.8 points without lowering the reflection gain.
This confirms the reward discount of Observation~\ref{prop:hazard} as a real cause of erosion, but removing it alone is not enough, as the drift and interference mechanisms contribute the rest of the initial-answer gain.

\paragraph{Closing the drift channel.}
As shown in Table~\ref{tab:ablation}, anchoring verified reference decisions at $t^\star$ raises initial-answer accuracy by 1.3 points while the reflection gain falls to $+0.1$, consistent with RL overwriting decisions the reference already gets right even under routed credit.
The anchor recovers these decisions without constraining the reasoning, and the floor multiplier, decision-local KL, and transition utility then partly restore the reflection gain.
Drift is therefore a distinct erosion channel that reward design alone cannot close, and anchoring verified decisions at the decision token closes it effectively.

\paragraph{Interference is most pronounced on long videos.}
As shown in Table~\ref{tab:ablation}, asymmetric gradient protection adds 1.2 points of initial-answer accuracy, 1.1 of final accuracy, and 2.0 on long videos.
Removing it from the full configuration lowers all three while \emph{raising} the reflection gain, the erosion signature predicted by Observation~\ref{prop:hazard}.
Gradient interference thus erodes the initial answer even when rewards and anchors are in place, and asymmetric protection is most valuable on long videos, where the initial answer is hardest to get right.
Across the progression, the two largest initial-answer gains, from the CE anchor and gradient protection, both carry through to final accuracy, supporting our premise that a stronger initial answer yields a stronger final one.

\subsection{Reflection on a Confident Initial Answer}\label{sec:analysis}

\begin{table}[t]
\centering
\caption{\textbf{Gate-conditioned outcomes on Video-MME} (\%). Each policy selects its own subsets; \seek{} and \trust{} are shown. Fix and break frequencies use all \seek{} questions as the denominator; the per-error repair rates in the text divide Fix by the initial-answer error rate on \seek{} questions.}
\label{tab:gate}
\small
\setlength{\tabcolsep}{3pt}
\fitresulttable{%
\begin{tabular}{lccccc}
\toprule
& \shortstack{\seek{}\\rate} & \shortstack{Acc\\\seek{}} & \shortstack{Acc\\\trust{}} & Fix & Break\\
\midrule
Outcome RL & 14.6 & 44.3 & 62.8 & 16.7 & 9.6\\
\method{} & 15.7 & \textbf{51.9} & \textbf{65.8} & 13.2 & 11.1\\
\bottomrule
\end{tabular}
}
\end{table}

\paragraph{Same corrector, more reliable initial answer.}
As shown in Table~\ref{tab:gate}, both gates send a similar share of questions to retrieval, yet the answers that \method{} trusts are 3.0 points more accurate.
On \seek{} questions, \method{} starts from a far more accurate initial answer (49.8\% vs.\ 37.2\%), while both correctors repair a similar fraction of initial errors (26.3\% vs.\ 26.6\%), far above the 6.2\% of \review{}, which re-examines the first look without new frames.
\method{} thus applies an equally effective corrector to a far more reliable initial answer, showing that its gains come from avoiding errors upfront rather than from repairing more of them.

\begin{figure}[t]
\centering
\includegraphics[width=\columnwidth]{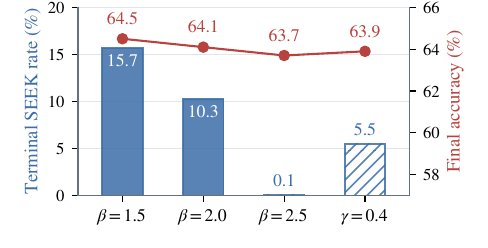}
\vspace{-6mm}
\caption{\textbf{Penalizing revision does not protect the initial answer.} Video-MME \seek{} rate, the share of questions routed to \seek{} (bars, left) and final accuracy (line, right), varying one penalty from the default $\beta=1.5$, $\gamma=0.2$.}
\vspace{-2mm}
\label{fig:penalty}
\end{figure}

\paragraph{Timid reflection and self-reported confidence are no substitute.}
As demonstrated in Figure~\ref{fig:penalty}, penalizing breaks more heavily suppresses reflection without strengthening the initial answer.
Raising $\beta$ from 1.5 to 2.0 cuts the \seek{} rate by a third and lowers final accuracy, and a larger $\beta$ or $\gamma$ nearly eliminates retrieval with further losses.
Making reflection timid thus gives up useful retrieval without making the initial answer more reliable.
Self-reported confidence is no substitute either. As reported in Table~\ref{tab:auc}, the gate's confidence predicts initial-answer correctness on Video-MME with an AUC of only 0.68, and the option belief over retrieved frames separates fixes from breaks at 0.53.
The model's own confidence is therefore too weak a signal to decide when to keep or overturn the initial answer.
Erosion is thus prevented by protecting the initial answer with verified correctness, not by making reflection timid.

\section{Conclusion}\label{sec:conclusion}

RL on answer-and-re-examine video loops can erode the initial answer, raising reflection gain while initial-answer accuracy stalls.
Following \emph{confident first, reflective second}, \method{} closes the incentive, drift, and interference channels with dedicated mechanisms, achieving the best results among compared 7B models on three video benchmarks and improving initial-answer and final accuracy over Outcome RL by 4.1 and 3.7 points on average.

\section*{Limitations}\label{sec:limitations}
We study one backbone and one answer-and-re-examine loop on long-video multiple-choice benchmarks, with one run per configuration; extending the analysis to other loops trained with final-outcome rewards, such as agentic temporal search, is future work.

\bibliography{custom}
\clearpage
\appendix

\section{Extended Related Work}\label{app:related}

\paragraph{Protecting one capability while learning another.}
Constrained RL enforces performance requirements with Lagrangian multipliers \citep{achiam2017cpo,tessler2019rcpo}, continual learning protects earlier tasks with parameter penalties or gradient constraints \citep{kirkpatrick2017ewc,lopezpaz2017gem,chaudhry2019agem}, PCGrad removes conflicts among task gradients \citep{yu2020pcgrad}, and self-reported confidence is known to be miscalibrated \citep{guo2017calibration,kadavath2022know}.
\method{} brings these ideas inside a single reflective policy, where answering and reflecting are coupled capabilities, and adapts each to that setting:
the constraint is a \emph{confidence floor} on the policy's own initial-answer accuracy, enforced by a dual multiplier on initial-answer credit;
consolidation targets only \emph{label-verified} reference decisions at the decision token, rather than decisions the model is merely confident in;
and gradient surgery is \emph{asymmetric}, so reflection yields to answering but never the reverse.

\section{Training Data and Supervised Initialization}\label{app:data}
\subsection{RL Question Pool}\label{app:rl_data}

We use the multiple-choice question pool released with Reflect-R1 \citep{reflectr1}.
The short split contains 30{,}974 questions on 15{,}147 videos from LLaVA-Video-178K, STAR, NeXT-QA, CLEVRER, and PerceptionTest.
These videos last 4--180 seconds, with a median duration of 32 seconds.
The long split contains 11{,}628 questions on 3{,}527 Panda-70M videos lasting 10 seconds to one hour, with a median duration of 17 minutes.
Its 21 question types include temporal gaps, action localization, object existence, duration, and temporal order.

\begin{table*}[t]
\centering
\caption{\textbf{Training sources and cached reference accuracy.} The cache includes 498 long-video predictions recomputed with a 12-frame budget.}
\label{tab:rl_data}
\small
\fitresulttable{%
\begin{tabular}{llrrc}
\toprule
Split & Source & Questions & Durations & $\Exp[r_b]$\\
\midrule
\multirow{5}{*}{Short} & LLaVA-Video-178K & 20{,}654 & 0--3\,min & 0.79\\
& STAR & 4{,}197 & & 0.38\\
& NeXT-QA & 2{,}297 & & 0.78\\
& CLEVRER & 1{,}919 & & 1.00\\
& PerceptionTest & 1{,}907 & & 0.66\\
\midrule
Long & Panda-70M & 11{,}628 & 10\,s--1\,h & 0.15\\
\midrule
Total & & 42{,}602 & & 0.58\\
\bottomrule
\end{tabular}
}
\end{table*}

\paragraph{Reference cache.}
Cached answers are greedy direct predictions from pretrained Qwen2.5-VL-7B-Instruct.
For 498 long-video questions without released predictions, we recomputed answers from 12 uniformly sampled frames.
Cache accuracy is 73.9\% on short questions, 15.3\% on long questions, and 57.9\% overall.

\paragraph{Overlap checks.}
Comparing training-video YouTube identifiers with the 900 Video-MME videos found no shared identifiers.
Exact question-string matches covered 12 generic templates, including questions about a video's main focus.

\subsection{Supervised Initialization}\label{app:sft_data}

Supervised trajectories teach an overview request, an initial answer, a gate decision, and a final answer.
Answer turns use \texttt{<think>\ldots</think>} followed by \texttt{<answer>\ldots</answer>}, and the gate outputs reasoning and the JSON fields \texttt{confidence}, \texttt{mode}, \texttt{queries}, and \texttt{num\_frames}.
Visual observations are inserted between assistant turns, and the supervised overview contains at most 60 frames.
All RL runs start from the same supervised checkpoint, and gate confidence targets are constructed from correctness labels.

\section{Implementation and Objective Details}\label{app:impl}
\subsection{Training Settings}

\begin{table*}[t]
\centering
\caption{\textbf{\method{} training and inference settings.}}
\label{tab:hparams}
\small
\setlength{\tabcolsep}{4pt}
\fitresulttable{%
\begin{tabular}{llp{8.8cm}}
\toprule
Group & Hyperparameter & Value\\
\midrule
Rollout & Preview (train / eval) & 2\,fps; $\le734$ frames, 10{,}240 visual tokens / $\le768$ frames, 15{,}360 tokens\\
& First look & One tool call by $\Spre$; $\le8$ cache frames selected by query relevance and temporal diversity\\
& Retrieval frames per round & Gate-selected, clipped to $[4,32]$ (default 8); no frame retrieved twice\\
& Retrieval rounds & \trust{}, \review{}: 0; \seek{}: 1; \multihop{}: $\le3$\\
& \multihop{} stopping & From round 2, sufficiency score $>0.7$ (top belief, margin, entropy)\\
& Option belief & SigLIP-so400m frame--option similarity, temperature 0.05\\
& Sampling & Temperature 1.0, top-$p$ 1.0 (vLLM)\\
& Maximum generated length & 512 tokens per stage in training, 1{,}024 in evaluation (greedy)\\
\midrule
Optimization & Optimizer & AdamW (momentum coefficients 0.9 and 0.999), no weight decay, gradient clipping 1.0\\
& Learning rate & $10^{-6}$, cosine decay, no warmup\\
& Batch & 2 questions $\times$ 8 trajectories per optimizer step\\
& Training steps & Up to 1{,}000\\
& Replay & Groups with all-zero advantages are replaced from a 16-group buffer\\
& Precision and system & bf16, DeepSpeed ZeRO-3 with CPU offload, gradient checkpointing, FlashAttention-2\\
& Hardware & 8$\times$H100 per run\\
\midrule
GRPO & Trajectories per question & 8\\
& Advantage normalization & Group centering without standard-deviation scaling\\
\midrule
Reward & Break / changed-wrong penalties & $\beta=1.5$, $\gamma=0.2$\\
& Gate effort cost & $0.05e$, with $e=0,0.25,0.5,1$ across the four modes\\
& Format reward weight & 0.2\\
\midrule
Confidence floor & Accuracy tolerance & $\epsilon=0.10$\\
& Dual step / maximum & $\eta_\lambda=0.1$, $\lambda_{\max}=5$\\
& Accuracy EMA decay & 0.99, initialized at the first batch estimate\\
& Initial multiplier & $\lambda=0$, updated after every rollout batch\\
\midrule
Anchoring & Cross-entropy weight & $\mu=1.0$\\
& Decision / later-stage KL & $\kappa_1=0.05$, $\kappa_2=0.005$, both to $\pi_{\mathrm{ref}}$\\
& KL estimator & $k_3$: $e^{d}-d-1$, $d=\log\pi_{\mathrm{ref}}-\log\pi_\theta$ clipped to $[-10,10]$\\
\midrule
Gradient & Reflection weight & $\alpha=1.0$\\
& Projection estimate & Last 16 trainable tensors with $\le10^6$ entries (top-layer norms and biases)\\
& Stabilizer / clipping & $\varepsilon_{\mathrm{proj}}=10^{-12}$; $\widehat\omega$ clipped to $[-10,0]$\\
\midrule
Sampling & Reference-error question weight & 2 in the full configuration\\
\bottomrule
\end{tabular}
}
\end{table*}

Losses are averaged over generated text positions in each stage, and the four gate effort costs correspond to \trust{}, \review{}, \seek{}, and \multihop{}, respectively.
First-look and retrieved frames come from a separate 2\,fps cache of at most 2{,}048 frames per video. Because the preview shares one pixel budget across its frames, it is coarse for long videos: at the 768-frame evaluation cap, each preview frame has at most 40 units of $28{\times}28$ pixels (about $177{\times}177$), and at most 192 units for short videos, whereas each cache frame has up to 256 units (about $448{\times}448$). The preview is used only by $\Spre$; the gate sees the first-look frames, and $\Srev$ sees the first-look and retrieved frames.
Apart from the decision-local and later-stage terms of Section~\ref{sec:anchor}, no sequence-wide KL is used, and format rewards are routed only to the gate and post-answer stage.

\paragraph{Outcome RL reward.}
Outcome RL scores each trajectory by its final correctness plus a correction bonus, $R_i=r_{f,i}+c\,\Ind[\fix_i]$ with $c>0$, and assigns the group-centered advantage $\ctr(R_i)$ to every generated token of the trajectory, without stage routing, floor multiplier, anchoring, or gradient protection; format rewards are the same as for \method{}.

\paragraph{Inference protocol.}
The preview is sampled at 2\,fps with at most 768 frames, and $\Spre$ adds a first look of at most eight cache frames; each retrieval round adds 4--32 gate-requested frames, with one round for \seek{} and at most three for \multihop{}.
SigLIP-so400m option beliefs let \multihop{} stop once evidence is sufficient, and evaluation decodes greedily.
Outcome RL uses this protocol unchanged, and Reflect-R1 is reproduced with its official pipeline.

\paragraph{Answer extraction.}
The returned answer is read from the initial-answer stage under \trust{} and from the post-answer stage otherwise; $\Acc_{\mathrm{pre}}$ always reads the initial-answer stage.
Within a stage, the option letter is taken from the answer tag, falling back to the last standalone option letter when the tag is missing or malformed; an output without any option letter is scored as incorrect.

\subsection{Reference Policies and the Confidence Floor}\label{app:floor}

The frozen pretrained model $\pi_{\mathrm{base}}$ supplies the cached direct answers $y_b$ and their correctness $r_b$, which define $J_{\mathrm{base}}$ and the cross-entropy targets in Equation~\ref{eq:dec}. The frozen supervised initialization $\pi_{\mathrm{ref}}$ supplies the KL targets and is also the starting RL policy. The cached answer is produced with a direct-answer prompt, whereas the initial answer is produced within the staged protocol.

The floor $J_{\mathrm{base}}-\epsilon$ is defined relative to the reference and is evaluated under the same training sampler $\mathcal D_w$ as $J_0$, so reference-error reweighting changes both quantities consistently (Appendix~\ref{app:accuracy_estimation}). We use $\epsilon=0.10$. The floor is fixed relative to the reference rather than to the best checkpoint reached during training. In training, $\lambda$ starts at 0, both moving averages are initialized at their first batch estimates, and $\lambda$ is updated after every rollout batch by Equation~\ref{eq:dual}.

\subsection{Estimating the Accuracy Signals}\label{app:accuracy_estimation}

\paragraph{Population quantities under the training sampler.}
For a training pool of $N$ labeled inputs $(x_n,y_n^\star)$, let $p_n=\Pr_{\mathcal D_w}(x_n)$ be the sampling probability of question $n$, with $\sum_{n=1}^{N}p_n=1$. Let $r_{b,n}$ denote its cached reference correctness. The two population quantities are
\begin{equation}
\begin{aligned}
J_0(\theta)
&=\sum_{n=1}^{N}p_n
\Pr_{\pi_\theta}(y_0=y_n^\star\mid x_n),\\
J_{\mathrm{base}}
&=\sum_{n=1}^{N}p_n r_{b,n}.
\end{aligned}
\label{eq:accuracy_population}
\end{equation}
The first quantity averages over sampled questions and initial-answer generation, and the second averages fixed cache entries under the same question distribution, so reference-error reweighting changes $p_n$ and therefore both quantities.

\paragraph{Batch accuracy estimates.}
At training step $t$, sample $B$ questions from $\mathcal D_w$ and generate $G$ trajectories per question using the current rollout policy. Let $r_{0,ji}^{(t)}$ be the raw initial-answer correctness of trajectory $i$ for sampled question $j$, and let $r_{b,j}^{(t)}$ be that question's cached reference correctness. Compute
\begin{equation}
\begin{aligned}
\bar J_{0,t}
&=\frac{1}{BG}\sum_{j=1}^{B}\sum_{i=1}^{G}r_{0,ji}^{(t)},\\
\bar J_{\mathrm{base},t}
&=\frac{1}{B}\sum_{j=1}^{B}r_{b,j}^{(t)}.
\end{aligned}
\label{eq:accuracy_batch}
\end{equation}
The first estimate uses only the initial answers, before any revision, and excludes the multiplier $1+\lambda$, group centering, transition utility, and auxiliary shaping. The reference estimate reads one cached value per question, so no reference-model generation is required, and since questions are already sampled from $\mathcal D_w$, the reference-error weights are not applied a second time.

\paragraph{Moving averages and update timing.}
Let $\nu=0.99$ be the accuracy EMA decay from Table~\ref{tab:hparams}, and let $\widehat J_{0,t}$ and $\widehat J_{\mathrm{base},t}$ denote the controller states available at the start of step $t$. After the policy update, incorporate the accuracies from that step's sampled trajectories:
\begin{equation}
\begin{aligned}
\widehat J_{0,t+1}
&=\nu\widehat J_{0,t}
+(1-\nu)\bar J_{0,t},\\
\widehat J_{\mathrm{base},t+1}
&=\nu\widehat J_{\mathrm{base},t}\\
&\quad +(1-\nu)\bar J_{\mathrm{base},t}.
\end{aligned}
\label{eq:accuracy_ema}
\end{equation}
The updated values enter Equation~\ref{eq:dual} to obtain the multiplier for the next batch (Algorithm~\ref{alg:rwr}). The statistics come from the batch already generated, so no second rollout is needed, and both EMA states and the multiplier are treated as constants during backpropagation.

\subsection{Stage Surrogates and Candidate Rewards}\label{app:surrogates}

Stage routing omits the downstream transition utility from the initial-answer score-function terms, so initial-answer priority is implemented through separate surrogates rather than an unbiased estimator of the joint gradient of Equation~\ref{eq:constrained}.
When all group rewards are equal, the task advantage is zero, while auxiliary terms and anchoring still contribute gradients.

For \trust{} trajectories, training still generates a post-answer candidate, and the gate and post-answer advantages use the candidate utility $u$ in every mode, whereas the utility of the returned answer is
\begin{equation}
u_{\mathrm{deploy}}=\Ind[m\ne\trust{}]u.
\label{eq:deployutility}
\end{equation}
The gate is thus trained on candidate utility even under \trust{}, where inference returns $y_0$.

\paragraph{Transition-penalty calibration.}
Let $p_{\fix},p_{\brk},p_{\cw}$ be transition frequencies under the initialization and rollout protocol.
For $p_{\brk}>0$, the mean raw transition utility is zero at
\begin{equation}
\beta^\star=\frac{p_{\fix}-\gamma p_{\cw}}{p_{\brk}}.
\label{eq:beta}
\end{equation}
Under the supervised initialization, $\beta^\star\approx1.43$, which motivates the default $\beta=1.5$; Figure~\ref{fig:penalty} reports the sensitivity.

\subsection{Reporting and Reproducibility Notes}\label{app:reporting}

Duration subtotals in Table~\ref{tab:training_duration} are computed on the same checkpoints as the overall scores, and Video-MME's three duration groups have equal size, so Overall is their mean.
Baseline scores in Table~\ref{tab:sota} are transcribed from Reflect-R1 \citep{reflectr1}; $^\ddagger$ marks Video-R1, whose published duration subtotals do not average to its published overall, a discrepancy present in the source table.
Baselines use the frame budgets and inference protocols of their original papers, so Table~\ref{tab:sota} compares complete systems; Table~\ref{tab:main} provides the controlled comparison under a single protocol.
Our reproduction of Reflect-R1 with its official pipeline reaches 62.6\% on Video-MME (Table~\ref{tab:training_duration}), 0.9 points below its published 63.5\%; Table~\ref{tab:sota} reports the published number.

\section{Detailed Experimental Results}\label{app:ablation_details}

\subsection{Video-MME Duration Breakdowns}

Table~\ref{tab:training_duration} reports Video-MME duration breakdowns for our training configurations and the official-pipeline Reflect-R1 reproduction, complementing Table~\ref{tab:main}.

\begin{table*}[t]
\centering
\caption{\textbf{Video-MME duration breakdowns for training comparisons.} All scores are final accuracy without subtitles. Reflect-R1$^\ast$ uses its official pipeline; the other rows use our staged protocol.}
\label{tab:training_duration}
\small
\setlength{\tabcolsep}{6pt}
\fitresulttable{%
\begin{tabular}{lcccc}
\toprule
Configuration & Short & Medium & Long & Overall\\
\midrule
Zero-shot, our protocol & 67.0 & 55.1 & 49.4 & 57.2\\
Reflect-R1 reproduction$^\ast$ & 74.7 & \textbf{62.3} & 50.7 & 62.6\\
Supervised initialization & 71.3 & 59.3 & 49.8 & 60.1\\
\quad + Outcome RL & 72.7 & 57.9 & 51.8 & 60.8\\
\rowcolor{gray!12}
\quad + \method{} & \textbf{75.0} & 61.2 & \textbf{57.3} & \textbf{64.5}\\
\bottomrule
\end{tabular}
}
\end{table*}

\subsection{Complete Training Progression}

Table~\ref{tab:ablation_full} lists all nine progressive configurations underlying Table~\ref{tab:ablation}, adding routed initial-answer credit, the CE anchor, the floor multiplier, decision-local KL, transition utility, and asymmetric gradient protection in sequence.
The final pair adds later-stage KL and reference-error reweighting together.

\begin{table*}[t]
\centering
\caption{\textbf{Progressive training configurations on Video-MME.} Components are added cumulatively except in the explicitly marked no-projection row. $^\S$The final two configurations add later-stage KL and double the reference-error sampling weight together.}
\label{tab:ablation_full}
\small
\setlength{\tabcolsep}{4pt}
\fitresulttable{%
\begin{tabular}{llcccc}
\toprule
Channel & Configuration & Overall & Long & $\Acc_{\mathrm{pre}}$ & $\Delta$\\
\midrule
-- & Outcome RL & 60.8 & 51.8 & 60.3 & $+0.5$\\
Incentive & + routed initial-answer credit & 61.6 & 52.7 & 61.0 & $+0.6$\\
Drift & \quad + cross-entropy anchor at $t^\star$ & 62.4 & 54.0 & 62.3 & $+0.1$\\
Incentive & \quad + floor multiplier & 62.2 & 53.6 & 62.4 & $-0.2$\\
Drift & \quad + KL at $t^\star$ & 62.3 & 54.1 & 62.3 & $0.0$\\
Incentive & \quad + transition utility & 62.3 & 53.8 & 62.0 & $+0.3$\\
Interference & \quad + asymmetric gradient protection & 63.4 & 55.8 & 63.2 & $+0.2$\\
All & \quad + later-stage KL / sampling$^\S$, no projection & 64.0 & 55.8 & 63.5 & $+0.5$\\
\rowcolor{gray!12}
All & \quad Full \method{}$^\S$ & \textbf{64.5} & \textbf{57.3} & \textbf{64.2} & $+0.3$\\
\bottomrule
\end{tabular}
}
\end{table*}

\subsection{Self-Reported Confidence as a Selection Signal}\label{app:auc}

Table~\ref{tab:auc} reports how well the model's own signals separate outcomes that matter for reflection on Video-MME.
Neither signal is reliable enough to decide when to keep or overturn the initial answer, which motivates protecting it with label-verified correctness (Section~\ref{sec:anchor}).

\begin{table}[t]
\centering
\caption{\textbf{Self-reported signals on Video-MME} (\method{}). AUC of each signal for the stated binary outcome; 0.5 is chance.}
\label{tab:auc}
\small
\setlength{\tabcolsep}{4pt}
\fitresulttable{%
\begin{tabular}{llc}
\toprule
Signal & Outcome predicted & AUC\\
\midrule
Gate confidence & Initial answer correct & 0.68\\
Retrieved-frame option belief & \fix{} vs.\ \brk{} (\seek{}) & 0.53\\
\bottomrule
\end{tabular}
}
\end{table}

\section{Training Algorithm}\label{app:algorithm}

Algorithm~\ref{alg:rwr} gives the \method{} update.
The projection coefficient is estimated on the tensor subset and detached before forming the full-model loss, and the moving averages and multiplier are updated after each policy step.

\begin{algorithm*}[t]
\caption{\methodlong{} (\method{})}
\label{alg:rwr}
\small
\begin{algorithmic}[1]
\Require Question pool $\mathcal D$ with $(y_b,r_b)$ cached by frozen $\pi_{\mathrm{base}}$; initialization $\theta_0$ and its frozen copy $\pi_{\mathrm{ref}}$
\Require Group size $G$; tolerance $\epsilon$; weights $\alpha,\mu,\kappa_1,\kappa_2$; dual settings $\eta_\lambda,\lambda_{\max}$
\Require Projection tensor subset $\mathcal P$; stabilizer $\varepsilon_{\mathrm{proj}}$
\Ensure Trained shared policy $\pi_\theta$
\State $\theta\gets\theta_0$; $\lambda\gets0$; initialize $\widehat J_0,\widehat J_{\mathrm{base}}$ at their first batch estimates
\For{each training batch}
    \State Sample questions $\mathcal B\sim\mathcal D_w$; use reference-error weight 2 in the full configuration
    \For{each question $x\in\mathcal B$}
        \State Sample $G$ trajectories $\tau_i=(\Spre^{(i)},\Sgate^{(i)},\Srev^{(i)})$ from $\pi_\theta$
        \State Extract $(y_{0,i},m_i,y_{f,i})$ and initial-answer positions $t_i^\star$
        \State Compute $r_{0,i}$ and $u_i$ using task labels and Equation~\ref{eq:utility}
        \State Assign $(1+\lambda)\ctr(r_{0,i})$ to $\Spre^{(i)}$ and $\ctr(u_i)$ to $\Sgate^{(i)},\Srev^{(i)}$
    \EndFor
    \State Add format and gate-effort shaping; construct clipped stage surrogates
    \State Compute $\mathcal L_{\mathrm{dec}}$ and $\mathcal L_{\mathrm{late}}$ using $\pi_{\mathrm{ref}}$
    \State $\mathcal L_{\mathrm p}\gets\mathcal L_{\mathrm{pre}}+\mathcal L_{\mathrm{dec}}$
    \State $\mathcal L_{\mathrm r}\gets\mathcal L_{\mathrm{gate}}+\mathcal L_{\mathrm{rev}}+\mathcal L_{\mathrm{late}}$
    \State $g_{\mathrm p}^{\mathcal P}\gets\nabla_{\theta_{\mathcal P}}\mathcal L_{\mathrm p}$;\quad
           $g_{\mathrm r}^{\mathcal P}\gets\nabla_{\theta_{\mathcal P}}\mathcal L_{\mathrm r}$
    \State $\widehat\omega\gets\mathrm{stopgrad}\!\left[
      \min\!\left(0,\dfrac{\langle g_{\mathrm r}^{\mathcal P},g_{\mathrm p}^{\mathcal P}\rangle}
      {\|g_{\mathrm p}^{\mathcal P}\|^2+\varepsilon_{\mathrm{proj}}}\right)\right]$
    \State $\mathcal L_{\mathrm{update}}\gets(1-\alpha\widehat\omega)\mathcal L_{\mathrm p}+\alpha\mathcal L_{\mathrm r}$
    \State Update $\theta$ with the configured optimizer using $\nabla_\theta\mathcal L_{\mathrm{update}}$
    \State Update $\widehat J_0,\widehat J_{\mathrm{base}}$ using Equations~\ref{eq:accuracy_batch}--\ref{eq:accuracy_ema}
    \State $\lambda\gets\mathrm{clip}\big(\lambda+\eta_\lambda(\widehat J_{\mathrm{base}}-\epsilon-\widehat J_0),0,\lambda_{\max}\big)$
\EndFor
\State \Return $\pi_\theta$
\end{algorithmic}
\end{algorithm*}

\section{Analysis of Asymmetric Gradient Protection}\label{app:proof}

\begin{proposition}[Fixed-batch protected descent]\label{prop:protect}
Assume that $\mathcal L_{\mathrm p}$ has a locally Lipschitz gradient, the sampled batch and advantages are fixed, $g_{\mathrm p}\ne0$, $\varepsilon_{\mathrm{proj}}=0$, and both gradients in Equation~\ref{eq:projection} are exact.
For a plain gradient step with size $\eta>0$ and $\alpha\ge0$,
\[
\mathcal L_{\mathrm p}(\theta-\eta g)-\mathcal L_{\mathrm p}(\theta)
\le-\eta\|g_{\mathrm p}\|^2+O(\eta^2).
\]
\end{proposition}
\begin{proof}
The exact projection satisfies $\langle g_{\mathrm p},g_{\mathrm r}'\rangle\ge0$.
A first-order expansion gives
\[
\begin{aligned}
&\mathcal L_{\mathrm p}(\theta-\eta g)-\mathcal L_{\mathrm p}(\theta)\\
&=-\eta\|g_{\mathrm p}\|^2
-\eta\alpha\langle g_{\mathrm p},g_{\mathrm r}'\rangle+O(\eta^2).
\end{aligned}
\]
The second first-order term is non-positive.
\end{proof}

The result follows from the asymmetric projection geometry of A-GEM \citep{chaudhry2019agem} and applies to the protected surrogate on a fixed batch.

\paragraph{Subset coefficient estimation.}
The coefficient $\widehat\omega$ is estimated on the last 16 trainable tensors with at most $10^6$ entries each, namely top-layer normalization weights and biases, which makes the two separate gradients cheap to compute; the final update is a single backward pass of Equation~\ref{eq:update} through all parameters.
With $\widehat\omega$ held constant during differentiation, the full-space inner product is
\[
\langle g_{\mathrm p},g_{\mathrm r}-\widehat\omega g_{\mathrm p}\rangle
=\langle g_{\mathrm p},g_{\mathrm r}\rangle-\widehat\omega\|g_{\mathrm p}\|^2.
\]

\paragraph{Numerical stabilizer.}
For $c=\langle g_{\mathrm r},g_{\mathrm p}\rangle<0$, a positive stabilizer leaves the residual
\begin{equation}
\langle g_{\mathrm p},g_{\mathrm r}'\rangle
=\frac{c\,\varepsilon_{\mathrm{proj}}}
{\|g_{\mathrm p}\|^2+\varepsilon_{\mathrm{proj}}},
\label{eq:residual}
\end{equation}
which vanishes as $\varepsilon_{\mathrm{proj}}\to0$; we use $\varepsilon_{\mathrm{proj}}=10^{-12}$.

\section{Qualitative Examples}\label{app:cases}

Figures~\ref{fig:case_180}--\ref{fig:case_568} show four Video-MME trajectories of the checkpoint reported in Table~\ref{tab:main}, in the format of Figure~\ref{fig:example}. Model outputs are verbatim and shortened with ``\ldots''; bold marks the key sentence. Thumbnails are the frames at the logged timestamps; $\Spre$ thumbnails are first-look frames. Timeline ticks mark first-look and retrieved frames, while the 2\,fps preview spans the whole video. Dashed boxes mark the evidence. We selected the three fixes from the 56 \fix{} cases among the 424 \seek{} questions. Each was kept only if the revision cites the newly retrieved frames, and those frames show what the revision claims. We excluded fixes that were correct only by guessing. The fourth example is one of the 47 \brk{} cases. Each figure also reports the Outcome RL baseline on the same question.

\paragraph{Fixes.}
In Figure~\ref{fig:case_180}, $\Spre$ misreads its own first look: the scoreboard at 44.0\,s lists CHN first, but $\Spre$ reports KOR and describes the skater at 48.5\,s as wearing a ``KOREA'' jersey, while the frame shows a Chinese skater. The gate re-reads the scoreboard, assigns confidence 0.12, and retrieves dense frames around the finish, which confirm CHN. In Figure~\ref{fig:case_766}, the 19-minute video shows the injured athlete only briefly, and no bib is legible in the coarse preview or the first look, so $\Spre$ guesses from the uniform color. The gate states exactly what is missing (``I cannot clearly read a bib number''), and all 12 retrieved frames fall in the 33-second treatment segment, where the bib ``377'' becomes visible. In Figure~\ref{fig:case_713}, $\Spre$ answers from a language prior, the documentary title \emph{Cocaine for Germany}. The gate notes that the video concerns Dutch cartels, and retrieval on the 42-minute video returns the full headline of the article that the first look showed only partially. In Figures~\ref{fig:case_180} and~\ref{fig:case_766}, the Outcome RL baseline instead trusts its wrong initial answer with confidence 0.90 and 0.70.

\paragraph{A failure.}
Figure~\ref{fig:case_568} shows a \brk{}. $\Spre$ infers the Los Angeles--New York route from the map at 44.5\,s and answers correctly (3,224 miles), but this answer rests on prior knowledge rather than an on-screen figure, so the gate correctly reports low confidence. The query ``total miles traveled by the main character'' retrieves a salient graphic that matches the option text: ``254,000 MILES'' over a map labeled ``1916'' and ``freight and passenger rail''. The revision reads the graphic correctly but attributes the 1916 network length to the traveler and switches to A. The perception is correct, but the referent is wrong. The prompt rule ``keep your initial answer unless the new evidence clearly contradicts it'' does not prevent this error, because the retrieved number matches an option verbatim and therefore appears decisive. The Outcome RL baseline retrieves the same frames and makes the same error. A grounded initial answer cannot be broken this way, but this initial answer was correct without on-screen grounding. This suggests that the \brk{} cost in Equation~\ref{eq:utility} is a backstop and does not replace checking whether retrieved evidence answers the question that was asked.

\begin{figure*}[t]
\centering
\includegraphics[width=\textwidth]{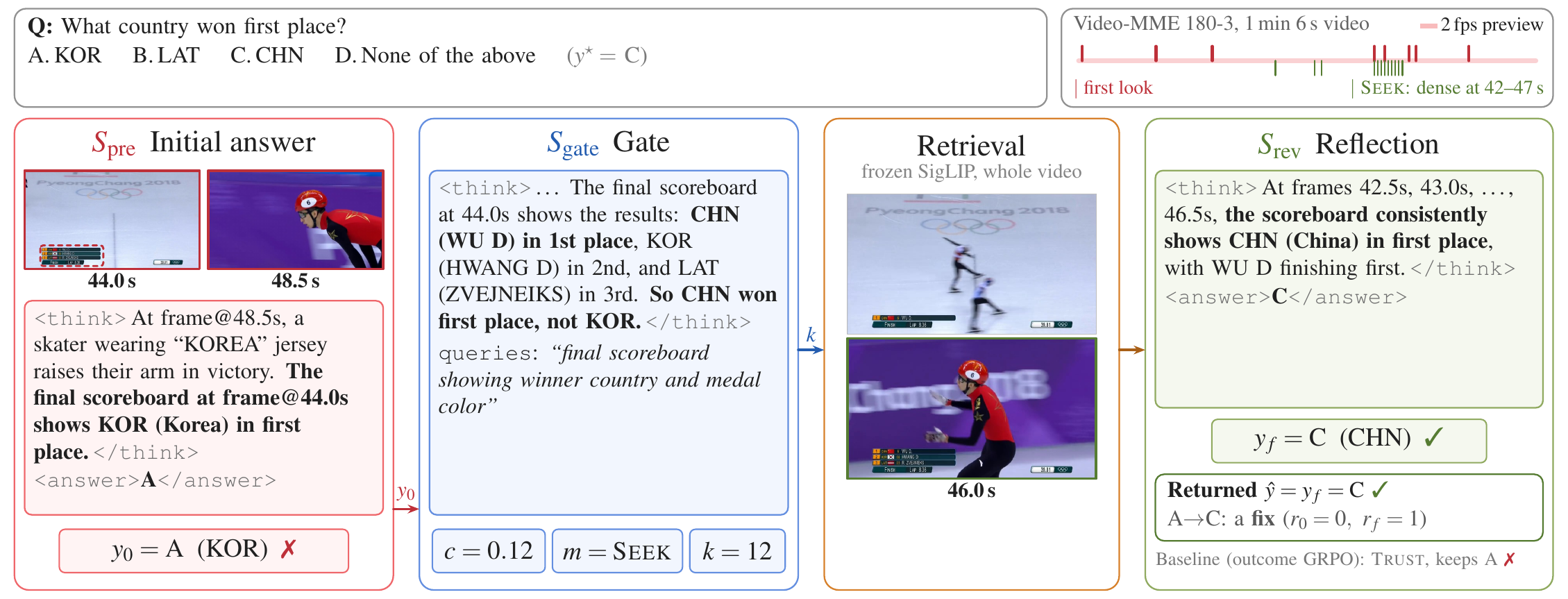}
\caption{\textbf{Fix: the gate catches a misread first-look frame} (Video-MME 180-3, 66\,s video). $\Spre$ reports KOR from a scoreboard that lists CHN first; the gate re-reads it, and dense retrieval around the finish confirms the revision to C.}
\label{fig:case_180}
\end{figure*}

\begin{figure*}[t]
\centering
\includegraphics[width=\textwidth]{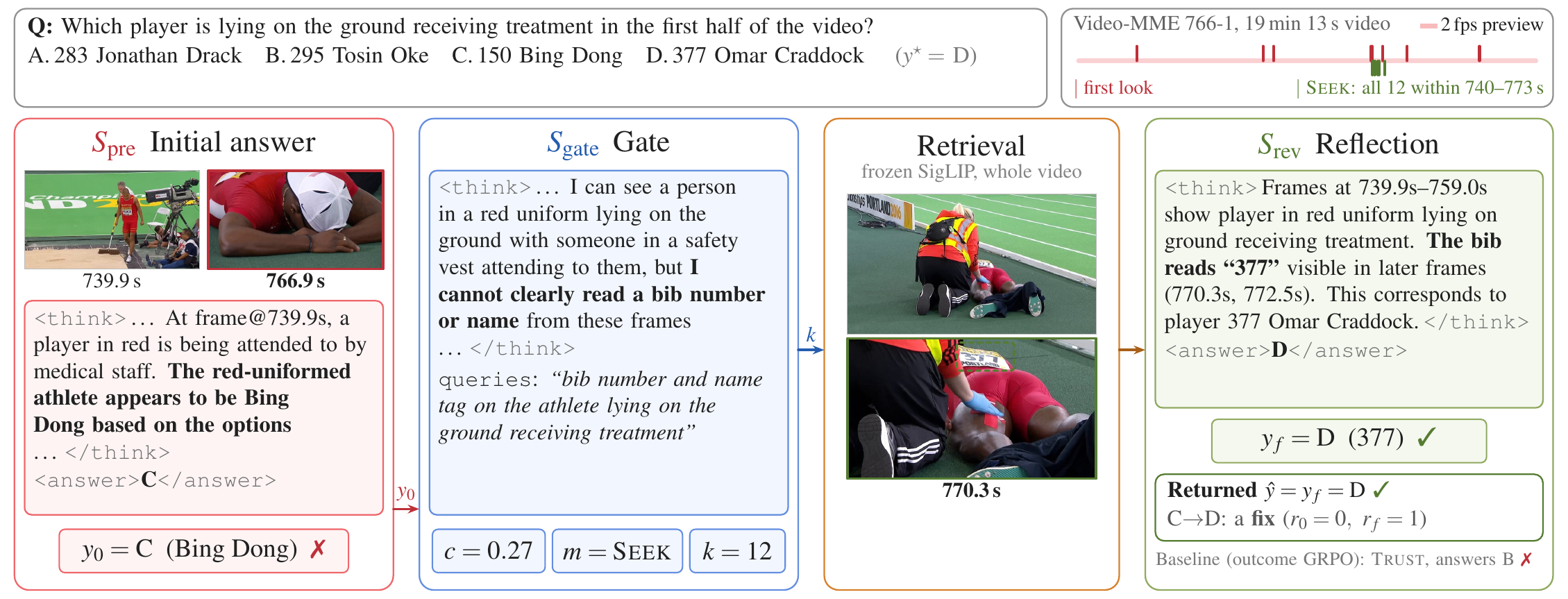}
\caption{\textbf{Fix: localizing a few seconds of evidence in a long video} (Video-MME 766-1, 19\,min). $\Spre$ guesses the athlete from the uniform color; the gate requests the bib, and all retrieved frames fall in the treatment segment, where the bib reads 377.}
\label{fig:case_766}
\end{figure*}

\begin{figure*}[t]
\centering
\includegraphics[width=\textwidth]{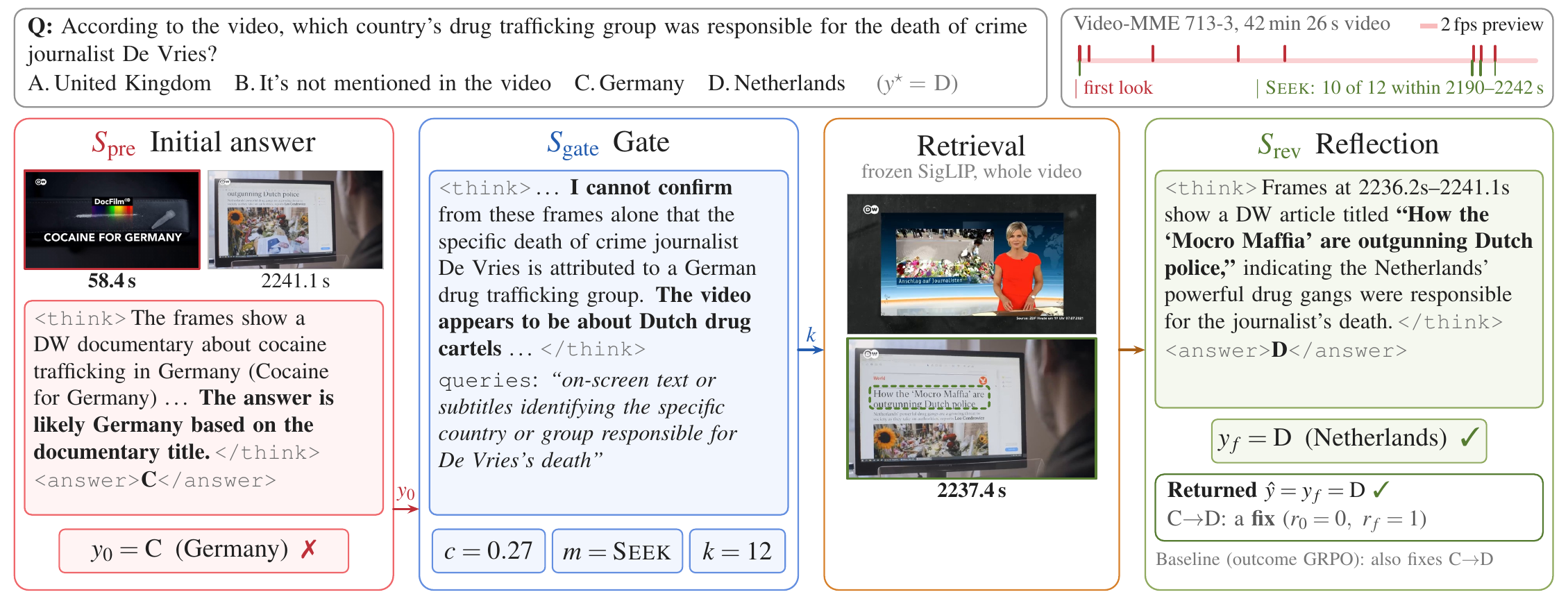}
\caption{\textbf{Fix: replacing a language prior with on-screen text} (Video-MME 713-3, 42\,min). $\Spre$ answers from the documentary title; retrieval returns the full headline attributing the threat to Dutch drug gangs.}
\label{fig:case_713}
\end{figure*}

\begin{figure*}[t]
\centering
\includegraphics[width=\textwidth]{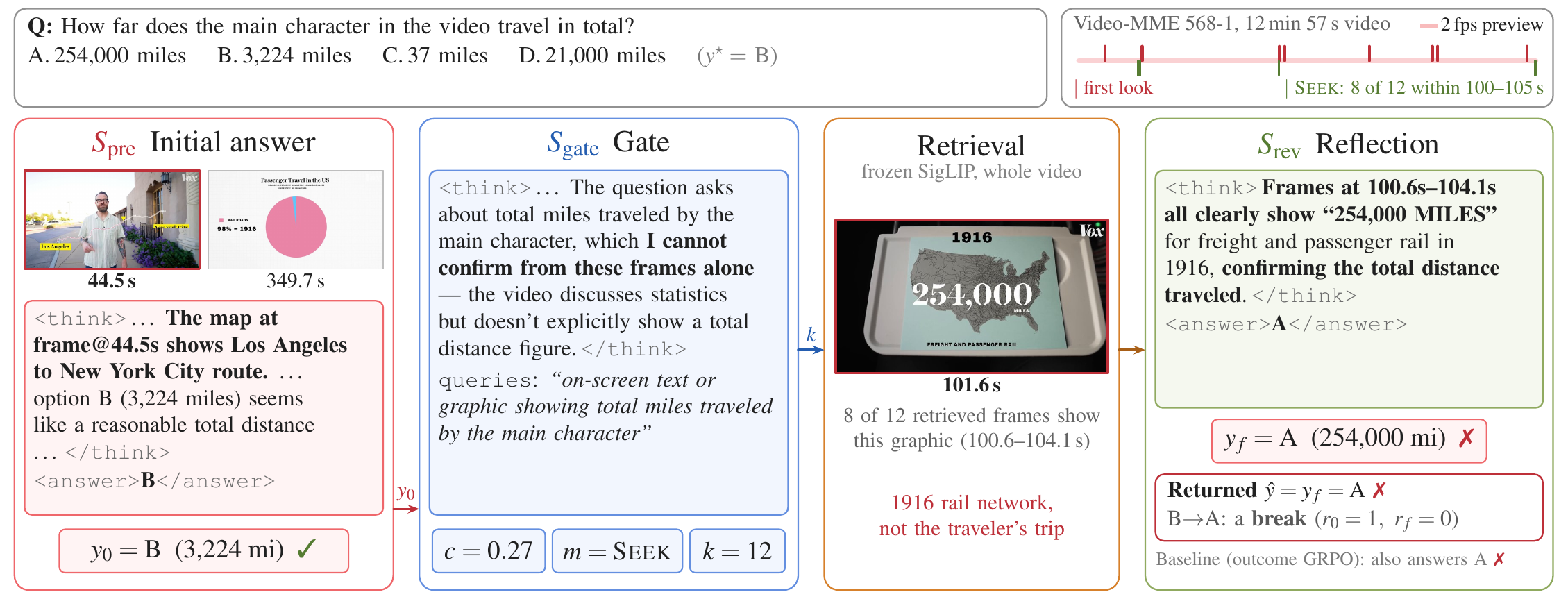}
\caption{\textbf{Failure: a correct initial answer broken by a salient but irrelevant number} (Video-MME 568-1, 13\,min). $\Spre$ is correct but not grounded in an on-screen figure; the query retrieves a 1916 rail-network statistic that matches option A verbatim, and the revision attributes it to the traveler.}
\label{fig:case_568}
\end{figure*}

\end{document}